\documentclass{article}
\usepackage{kr}

\usepackage{times}
\usepackage{soul}
\usepackage{url}
\usepackage[hidelinks]{hyperref}
\usepackage[utf8]{inputenc}
\usepackage[small]{caption}
\usepackage{graphicx}
\usepackage{amsmath}
\usepackage{amsthm}
\usepackage{booktabs}
\usepackage{algorithm}
\usepackage{algorithmic}
\usepackage{amssymb,mathtools,array,nicefrac,bm} 
\usepackage{xcolor,soul} 

\newtheorem{lemma}{Lemma}
\newtheorem{definition}{Definition}
\newtheorem{example}{Example}
\newtheorem{corollary}{Corollary}
\newtheorem{proposition}{Proposition}
\newtheorem{extraproposition}{Proposition}

\newcommand{\calE}{\bm{\mathcal{E}}}
\newcommand{\calEQ}{\bm{\mathcal{E}}^Q}
\newcommand{\E}{\bm{E}}
\newcommand{\A}{\bm{A}}

\newcommand{\calJ}[1]{\bm{\mathcal{J}}^{#1}}
\newcommand{\topology}[1]{\bm{\tau}_{#1}}
\newcommand{\Bel}[0]{\ensuremath{\mathrm{Bel}}}
\newcommand*{\medcap}{\mathbin{\scalebox{1.4}{\ensuremath{\cap}}}}
\newcommand*{\medcapeq}{\mathbin{\scalebox{1.1}{\ensuremath{\cap}}}}
\newcommand*{\medcup}{\mathbin{\scalebox{1.4}{\ensuremath{\cup}}}}
\newcommand*{\medcupeq}{\mathbin{\scalebox{1.1}{\ensuremath{\bigcup}}}}

\newcolumntype{C}[1]{>{\centering\arraybackslash}m{#1}}
\newcolumntype{R}[1]{>{\raggedleft\arraybackslash}m{#1}}

\newcounter{exampleRemember}

\title{A Belief Model for Conflicting and Uncertain Evidence \\ -- Connecting Dempster-Shafer Theory and the Topology of Evidence}

\author{%
Daira Pinto Prieto\\
Ronald de Haan\\
Ayb\"uke \"Ozg\"un
\affiliations
Institute for Logic, Language and Computation, University of Amsterdam
\emails
\{d.pintoprieto, r.dehaan, a.ozgun\}@uva.nl
}

\newcommand{\fullversiononly}[1]{#1}
\newcommand{\shortversiononly}[1]{}

\newcommand{\aybuke}[1]{\textcolor{black}{#1}}

\begin{document}

\maketitle

\begin{abstract}
One problem to solve in the context of information fusion, decision-making, and other artificial intelligence challenges is to compute justified beliefs based on evidence. In real-life examples, this evidence may be inconsistent, incomplete, or uncertain, making the problem of evidence fusion highly non-trivial. In this paper, we propose a new model for measuring degrees of beliefs based on possibly inconsistent, incomplete, and uncertain evidence, by combining tools from Dempster-Shafer Theory and Topological Models of Evidence. Our belief model is more general than the aforementioned approaches in two important ways: (1) it can reproduce them when appropriate constraints are imposed, and, more notably, (2) it is flexible enough to compute beliefs according to various standards that represent agents' evidential demands. The latter novelty allows the users of our model to employ it to compute an agent's (possibly) distinct degrees of belief, based on the same evidence, in situations when, e.g, the agent prioritizes avoiding false negatives and when it prioritizes avoiding false positives. Finally, we show that computing degrees of belief with this model is \#P-complete in general.
\end{abstract}

\section{Introduction}
In everyday life, we constantly receive evidence from outside world - by, e.g, following various news channels, observing our close environment, taking measurements, receiving testimony from others - and form beliefs about the world around us based on the pieces of evidence we have gathered and processed. Nowadays, we also aim at developing artificial agents, such as fully autonomous or driver-assist cars, which can gather evidence from their environment and merge it in the ``right'' way, to form {\em evidentially grounded, justified} beliefs so-to-speak, so as to avoid harming their environment.

The question of how an (artificial or human) agent should or can merge evidence to form evidentially justified beliefs has received growing attention in many disciplines, spanning computer science, artificial intelligence, decision theory, confirmation theory as well as both formal and traditional epistemology. One of the main difficulties behind answering this question seems to consist in the various essential features of the pieces of evidence collected from several sources. For example, of particular importance in this paper, different sources might provide {\em possibly false and mutually contradictory evidence}, the agent might have {\em varying degrees of uncertainties} about the pieces of evidence received (due to, e.g., the agent's degree of confidence in the source of evidence or the agent's uncertainties about how to interpret the evidence). Moreover, how an agent merges a (possibly mutually contradictory and uncertain) body of evidence and their degree of belief in a proposition based on that evidence also depend on the agent's {\em evidential demand} in a given context. To explain this notion of evidential demand, let us consider the following case. A cautious agent who is forming beliefs about a highly risky situation (a fully autonomous car running into a pedestrian on the road or a doctor who is about to offer a treatment for a serious disease), e.g., might benefit from taking every little piece of evidence they have into account in order to avoid false negatives, even though some evidence pieces might contradict each other. A driver-assist car, on the other hand, can be programmed to alert the actual driver of the car about a suspected danger only when the pieces of evidence that do not conflict with any other piece support, with relatively high certainty,  that there is a danger, in order to avoid false positives and needlessly distracting the driver.

In this paper, we propose a model for measuring degrees of beliefs for an agent who possesses a body of possibly {\em mutually contradictory, incomplete,} and {\em uncertain} pieces of evidence and whose {\em evidential demands} may change depending on the context. To do so, we employ tools and methods from the Dempster-Shafer Theory (DST) of belief functions \cite{SHAFER1976,YagerLiu08} and Topological Models of Evidence (TME) (introduced in  \cite{BALTAG16,OZGUN2017,BALTAG2022} and inspired by the evidence models of \cite{BENTHEM2011}). When combined, these two theories provide us with richer formal framework as they complement each other's weaknesses. While DST offers a  quantitative framework that can  represent agents' uncertainties about evidence and generate degrees of belief based on uncertain evidence, it is known to lead to counter-intuitive results, such as Zadeh's example \cite{ZADEH1986}, in presence of conflicting evidence with high uncertainty. This problem has motivated many variations of the Dempster rule of combination to compute degrees of belief in presence of conflicting evidence, each one with different advantages and disadvantages~\cite{SENTZ2002,DENOEUX2008}. %
TME, on the other hand, represent evidence purely qualitatively as sets of possible states (thus taking every piece of evidence to be equally (un)certain) but are naturally suited to model formation of consistent beliefs based on possibly mutually inconsistent body of evidence \cite{BENTHEM2011,OZGUN2017,BALTAG2022}. They moreover enable a fine-grained mathematical representation of various notions of evidence, such as {\em basic evidence}, {\em combined evidence}, {\em argument}, and {\em justification}, which will help us formalize agents' evidential demands.


The paper is organized as follows. Section \ref{sec:preliminaries} presents the required technical and conceptual preliminaries for TME and DST. In Section \ref{sec:model} we introduce our multi-layer belief model in three stages and demonstrate the role of each layer on a running example. In Section \ref{sec:assessment} we evaluate our model. We first compare it with the belief models of DST and TME by showing that ours can reproduce them when appropriate constraints are imposed (Section \ref{subsec:comparison}). Second, we prove that the computational complexity of computing degrees of belief using our model is $\#P$-complete (Section \ref{subsec:complexity}). Section \ref{sec:conclusions} summarizes our main contributions and lists a few directions for future research. \shortversiononly{Some proofs are omitted or summarized in this version due limited space. They are presented in detail in the full version provided as supplementary material.}

\section{Preliminaries}\label{sec:preliminaries}

\begin{definition}[Topological Space]\label{defn:top:space}
A {\em topological space}  is a pair $(S, \topology{})$  , where $S$ is  a nonempty set  and $\topology{}$ is a family  of subsets of $S$ such that  $S, \emptyset\in\topology{},$ and $\topology{}$ is closed under finite intersections and arbitrary unions.
\end{definition}

The set $S$ is  a \emph{space}; the family $\topology{}$ is called a {\em topology} on $S$.  The elements of $\topology{}$ are called \emph{open sets} (or \emph{opens}) in the space. A topology $\topology{}$ on a space $S$ can also be generated from an arbitrary subset of $2^S$. Given any family $\calE\subseteq 2^S$ of subsets of $S$, there exists a unique, smallest topology $\topology{\calE}$ with $\calE\subseteq \topology{\calE}$ \cite[Theorem\ 3.1, p. 65]{DUGUNDJI1965}. The family $\topology{\calE}$ consists of $\emptyset$, $S$, all finite intersections of $\calE$, and all arbitrary unions of these finite intersections.~$\calE$ is called a \emph{subbasis} for $\topology{\calE}$ and  $\topology{\calE}$ is said to be \emph{generated} by $\calE$.~The set of finite intersections of members of $\calE$ forms a topological basis for $\topology{\calE}$. A final important topological notion we employ in this paper (and important in topological representations of evidence) is {\em denseness}. A set $P\subseteq S$ is called {\em dense} in $S$ (with respect to topology $\topology{}$) if and only if $P\cap T\not =\emptyset$ for all $T\in \topology{}$ such that $T\not =\emptyset$. When it is contextually clear, we call a set {\em dense} only and avoid mention of the relevant space and its respective topology.

We call a tuple $(S, \calE)$ a {\em qualitative evidence frame}\footnote{Our qualitative evidence frame is a version of the so-called uniform evidence models introduced by \cite{BENTHEM2011}, and later developed into a topological version in \cite{BALTAG16,OZGUN2017,BALTAG2022}. The essential differences are that these sources impose the constraints $\emptyset\not \in \calE$ and $S\in\calE$ whereas we do exclude $S$ from $\calE$, and they do not require $S$ to be finite.  The reason for the finiteness constraint on $S$ is to avoid complication around assigning mass and probability values to infinite sets. Moreover, we impose $S\not \in \calE$ because, following DST, the mass values assigned to $S$ will represent degrees of uncertainties about evidence, not the degree of certainty in $S$ (which is always $1$).   Besides this, we follow the terminology of \cite{BALTAG2022}.}%
, where $S$ is a finite nonempty set of {\em possible states} and $\calE$ is a nonempty subset of $2^S$ called the {\em set of basic pieces of evidence} such that $\emptyset\not\in \calE$ and  $S\not \in \calE$. Following \cite{BALTAG2022}, the topology $\topology{\calE}$ generated by $\calE$ is called the {\em evidential topology}.  We think of a subbasis $\calE$ as a collection of propositional contents of pieces of evidence that are {\em directly} obtained by an agent via, e.g., observation, measurement, testimony from others, experiments, etc. They are the basic pieces of evidence in this sense.  The elements of the evidential topology $\topology{\calE}$  represent the propositional content of pieces of evidence that the agent put together by taking the intersections and unions of directly observed evidence (i.e., using the terminology of \cite{BALTAG2022}, the elements of $\tau_{\calE}$ represent {\em combined} evidence).\footnote{The treatment of open sets as pieces of evidence dates back to \cite{Troelstra1988} and adopted from topological semantics for intuitionistic logic. It moreover has applications in domain theory \cite{VICKERS1989} and formal learning theory \cite{KELLY1996}. We refer to \cite{OZGUN2017} for a more detailed explanation of the epistemic interpretations of topological spaces.}

To clarify, by propositional content of a piece of evidence, we mean the information provided by a piece of evidence without regard to how uncertain that piece of evidence is.  Understood this way, the topological framework is purely qualitative and evaluates every piece of evidence the agent has on a par with respect to uncertainty. We sometimes only say `a piece of evidence' to refer to the propositional content of that piece of evidence, the context will hopefully make our usage clear. Moreover, we say `a {\em basic} piece of evidence' only when we want to emphasize that it is an element of $\calE$.

This topological framework is not only rich enough to differentiate direct evidence (elements of a subbasis $\calE$) from combined evidence (elements of $\topology{\calE}$), but it can also discern notions argument and justification, proving a finer-grained scale of notions of evidence. In this sense, not every piece of evidence or argument constitutes a piece of justification for agents' beliefs. Underlying intuition here is that an agent might have high or low demands for what constitutes justification for their beliefs. E.g., observing something alone might not be good enough, the agent might also require consistency with all the other observations they have made (this is the rough intuition behind the topological notion of justification to be presented below). Intuitively speaking, we take justifications to be special kinds of arguments that the agent uses to support their beliefs and the form of such justifications depends on the agent's evidential demands. We elaborate more on this in Section \ref{subsec:qualitative}.\footnote{This perspective on justification substantially differs from the understanding of justification in traditional epistemology, where the main interest lies in defining the (or a good) notion of justification for belief,  or answering the question ``what justifies belief?". Our use of the term justification here is more pragmatically motivated and it is simply intended to discern an agent's any evidence-based argument from the ones they actually see fit to support their beliefs.} 

Given a qualitative evidence frame $(S, \calE)$, a proposition $P\subseteq S$, and a piece of evidence $E\in \tau_{\calE}$, $E$ {\em supports} $P$ iff $E\subseteq P$. An {\em argument for} $P$ is an element $T\in \topology{\calE}\setminus \{\emptyset\}$ such that $T\subseteq P$ \cite[p. 512]{BALTAG2022}. Topologically, an argument is just a nonempty open set in a topology. We say that two pieces of evidence $E, E'\in \topology{\calE}$ are {\em mutually inconsistent} if $E\cap E' =\emptyset$, and mutually consistent otherwise. Notice that  $\topology{\calE}$ (as well as $\calE$) can host mutually inconsistent pieces of evidence (we did not impose any constraints on $\calE$ to eliminate such cases). We call a proposition $P\subseteq S$ {\em consistent with} $\topology{\calE}$ if and only if $P\cap T\not =\emptyset$ for all $T\in \topology{\calE}\setminus \{\emptyset\}$. That is, $P\subseteq S$ is consistent with $\topology{\calE}$ if and only if $P$ is dense in $S$. Against this background, \cite{BALTAG2022} propose to define (evidential) justification as dense open subsets of a topology, and justified belief as those propositions that have dense open subsets, i.e., propositions that are supported by justifications. More formally, a proposition $P$ is believed if and only if there is an argument $T$ for $P$, that is, $T\in \topology{\calE}\setminus \{\emptyset\}$ and $T\subseteq P$, such that $T$ is consistent with any piece of evidence in $\topology{\calE}\setminus \{\emptyset\}$, i.e., $T$ is dense in $S$ w.r.t.\  $\topology{\calE}$. This definition of belief entails that the agent always has consistent beliefs, even when the belief is formed based on a set of possibly mutually inconsistent pieces of evidence (for details, please see \cite{BALTAG2022}).

The topological framework, however, is lacking the ingredients to talk about agents' potential uncertainty about the pieces of evidence gathered. To be able to account for such quantitative aspects of evidence, we combine the topological framework with the Dempster-Shafer Theory (DST), which employs \emph{belief functions} to assign a degree of belief to each proposition - i.e., to each subset $S$ of possible states. 
In the remaining of this section, we introduce the required preliminaries from DST.


\begin{definition}[Basic probability assignment]
Given a finite set of possible states~$S$, a \emph{basic probability assignment} over the set~$S$
is a function~$m : 2^S \to [0,1]$ such that~$m(\emptyset) = 0$
and~$\sum_{A\subseteq S}m(A) = 1$.
\end{definition}

A basic probability assignment expresses the degree of certainty of each subset of $S$ according to the evidence. Therefore, given a basic probability assignment~$m$ and a proposition~$P\subseteq S$, $m(P)>0$ is not evidence against its complement $\neg P$. In addition, $m(S)$ represents the uncertainty of the evidence modeled by the basic probability assignment~$m$. Given a collection of basic probability assignments, their certainty values can be merged and produce a single basic probability assignment by applying the so-called {\em Dempster's rule of combination}.

\begin{definition}[Dempster's rule of combination (DRC)]
Let~$m_1$ and~$m_2$ be basic probability assignments over the same finite set~$S$ of possible states and $A_1,\dotsc,A_{k}$ and~$B_1,\dotsc,B_{\ell}$ all subsets of $S$ such that $m_1(A_i)\neq0$ and $m_2(B_i)\neq0$, respectively. 
Moreover, suppose that~$\sum_{A_i\cap B_j = \emptyset}m_1(A_i)m_2(B_j)<1$.
Then the following basic probability assignment~$m$, 
also denoted by~$m_1~\oplus~m_2$,
is the result of applying \emph{Dempster's rule of combination} to~$m_1$ and~$m_2$:
$m(\emptyset) = 0$ and
$m(C) = \nicefrac{\sum_{A_i\cap B_j = C}m_1(A_i)m_2(B_j)}{K}$,
where~$K$ is the normalization
factor~$1-\sum_{A_i\cap B_j = \emptyset}m_1(A_i)m_2(B_j)$,
for all nonempty sets~$C \subseteq S$.
\end{definition}

\begin{definition}[Belief function]\label{def:belief-function}
Given a finite set of possible states~$S$, a \emph{belief function}
is a function~$\Bel : 2^S \to [0,1]$ such that~$\Bel(\emptyset) = 0$, $\Bel(S) = 1$ and~$\Bel\big(\medcup_1^n A_i\big) \geq \sum_{\emptyset\neq I \subseteq \{1,\dots,n\}}(-1)^{|I|+1}\Bel\big(\medcap_{i\in I}A_i\big)$.
\end{definition}

Given a basic probability assignment $m$, the function $\Bel: 2^S \to [0,1]$ such that $\Bel(P)=\sum_{A\subseteq P} m(A)$ is a belief function \cite{SHAFER1976}. Actually, the main focus of this theory is on belief functions that can be computed from basic probability assignments by using DRC. They are called {\em support functions}; and when these basic probability assignments only give non-null value to one proposition $A \subset S$, they are called {\em simple support functions}.

While DST uses basic probability assignments to represent  the evidence
that is available, we will represent them by a tuple $(S,\calEQ)$ called \emph{quantitative evidence frame}, where $S$ and $\calE$ are as in a qualitative evidence frame, and $\calEQ$ is a nonempty subset of $\calE\times (0,1)$. We choose the open interval $(0,1)$ to avoid considering exceptions in the following sections. However, this assumption is not a significant limitation since the aim of the approach is to model uncertainty. For any element $(E, p)\in \calEQ$, $E$ represents the propositional content of the evidence and $p$ is its degree of certainty. Given a pair $(E, p)\in \calEQ$, the value $1-p$ represents the uncertainty of the given piece of evidence (and not the certainty of $S\setminus E$). In DST terminology, the set $\calEQ$ can be interpreted as a set of simple support functions where every element $(E,p) \in \calEQ$ represents a simple support function $m$ such that $m(E) = p$.

Given a quantitative evidence frame $(S,\calEQ)$, our ultimate goal is to define a belief function in the sense of Definition \ref{def:belief-function} and which admits as input not only a body of possible mutually inconsistent and uncertain evidence $\calEQ$  but also the evidential demands of the agent. To this end, we will define a mass function and a basic probability assignment. 

\begin{definition}[Mass function]\label{defn:mass}
Given a nonempty set~$X$, a \emph{mass function} over the set~$X$
is a function~$m : 2^X\to [0,1]$ such that~$\sum_{A\subseteq X}m(A) = 1$.
\end{definition}

Aiming to help the reader to follow the text smoothly, we have set a notation code. Lowercase letters refer to possible states. Uppercase letters are used to specify sets of possible states. In particular, we will use $E$ to represent basic pieces of evidence, $T$ to represent the elements of a topology, and $S$ to represent the set of all the possible states. Sets of the previous sets are named in bold capital letters. Some examples are $\E$ to specify any set of pieces of evidence and $\calE$ to specify the set of all the pieces of evidence. Finally, the subsets of $2^{\calE}$ are denoted by blackboard bold capital letters (such as $\mathbb{M}$). In the previous definition, we made an exemption to this rule since $X$ represents any set.


\section{Multi-Layer Belief Model}\label{sec:model}

Our belief model is built on three different layers. The first layer, called the {\em qualitative layer}, works with the propositional content of the basic pieces of evidence (i.e., elements of $\calE$) and identifies a set of justifications that represents the agent's evidential demands. The second layer, called the {\em quantitative layer}, focuses on degrees of support supplied by uncertain evidence and transforms the degrees of uncertainties of the basic pieces of evidence into a mass function defined over the set of basic pieces of evidence $\calE$. The last layer, called the {\em bridging layer}, connects the values of the mass function obtained in the second layer to the justifications of the first one. As a result, we obtain a belief function which is able to compute degrees of belief for an agent according to their evidential demands and based on a possibly mutually contradictory, incomplete, and uncertain body of evidence.

As mentioned, this model aims to be a tool to combine pieces of evidence and compute  degrees of belief based on (different ways of combining) evidence, so we will build a running example inspired by systems where there is an evidence fusion problem. In addition, the novelty of this model is to consider agents' evidential demands for the computation of the degrees of belief, so it may be especially interesting for those systems which try to avoid false negatives or false positives, depending on the situation at hand.\footnote{Note that {\em evidential demand} refers to a constraint given by agents independently of the sources of evidence. It should not be confused with agents' degrees of trust on the sources. In particular, evidential demand does not refer to agents' prior beliefs.} 
To illustrate, consider a fully autonomous car and a driver-assist car, both of which are equipped with the same sensors to observe their immediate environment and collect evidence of varying degrees of uncertainty. However, their goals, therefore, their evidential demands, may differ. The autonomous car is intended to be safe and stop at the slightest evidence of danger. The driver-assist car, on the other hand, may have higher evidential demands - such as mutual consistency of evidence - to act since a too-often warning alert may distract the driver. In this section, we explain our three-layer construction and clarify the notions we have just introduced by using the following example.

 \begin{example}
Let A be a fully autonomous car and B a driver-assist car. They both detect an object crossing the road in front of them.  Due to the distance between the cars and the object, the desirable outcome would be to stop  only if it is a static object or a pedestrian. Both cars collect the following pieces of evidence: 
\begin{align*}
    &(E_1,0.9) := \text{`It is a dynamic object' with } 90\% \text{ certainty.} \\
    &(E_2,0.75) := \text{`It is a motorbike' with } 75\% \text{ certainty.} \\
    &(E_3,0.45) := \text{`It is a pedestrian' with } 45\% \text{ certainty.} \\
\end{align*}
This evidence could come from a situation where one classifier of the system has been trained exclusively with motorbikes, another one exclusively with pedestrians, and what the cars have in front of them is a motorbike with a driver who is not wearing a helmet (so the sensors identify human features).

The set of possible states for cars is denoted by $S=\{sp, dp, do, so, dm, sm\}$, where `$sp$' represents the possible state where the object is a static pedestrian and, similarly, `$dp$' refers to `dynamic pedestrian', `$do$' refers to `other dynamic object',  `$so$' to `other static object', `$dm$'  to `dynamic motorbike', and  `$sm$' to `static motorbike'. 
The propositional contents of the directly observed evidence are $E_1 = \{dp, dm, do\}$, $E_2 = \{dm, sm\}$ and $E_3 = \{dp, sp\}$, respectively, and the set of basic evidence sets is $\calE=\{E_1, E_2, E_3\}$. This set generates the following evidential topology:
\begin{equation*}
\begin{split}
    \topology{\calE} =& \big\{\emptyset,E_1, E_2, E_3, \{dp\}, \{dm\}, \{dp,dm\},\\
    &\{sp,dp,do,dm\}, \{dp, do, dm, sm\},\\
    &\{sp,dp,dm,sm\}, S \big\}.
\end{split}
\end{equation*} 
\end{example}

\subsection{Qualitative Layer}\label{subsec:qualitative}

To recall, given a qualitative evidence frame $(S, \calE)$, the set $\topology{\calE}{\setminus}\{\emptyset\}$ represents the set of arguments available  to the agent. Crucially, this model allows us to go one step further and distinguish between {\em arguments} and {\em justifications} available to the agent (as can be done in the purely topological framework \cite{BALTAG2022}). While in the original topological framework justifications are defined to be dense opens, we here take a more flexible approach by allowing any argument to be a potential piece of justification for belief.  We take that what constitutes a piece of justification for an agent - how the agent uses the evidence they have - may depend on various factors such as the question in hand, how risky the situation is, how cautious the agent should be, and the goal of the agent, among others. These factors, in turn, determine the agent's evidential demands. 

To distinguish arguments from justifications and formalize an agent's evidential demands, we use the notion of {\em frame of justification}: given $(S, \calE)$, a frame of justification $\calJ{}$ is just a subset of $\topology{\calE}$. Depending on the situation modelled, one can think of natural constraints on frames of justifications. In this paper, due to limited space, we focus on two kinds of frame of justifications, first of which is inspired by DST and the second one by the aforementioned topological framework. Given a qualitative evidence frame $(S, \calE)$:

\begin{enumerate}
    \item The {\em Dempster-Shafer frame of justification}, denoted by $\calJ{DS}$, is the set of all arguments, that is, $\calJ{DS} =\topology{\calE}{\setminus} \{\emptyset\}$. This frame represents agents with very low evidential demands. For these agents, having an argument for $P$ among their evidence is enough to justify $P$, regardless whether the argument contradicts with the other available arguments.
    \item The {\em strong denseness frame of justification}, denoted by $\calJ{SD}$, is the set of all arguments consistent with $\topology{\calE}$, that is, the set of all dense sets in $S$. 
    This frame represents agents with high evidential demands. They form degrees of beliefs only based on arguments which do not contradict with, i.e., cannot be refuted by, any other argument. Consequently, they form degrees of belief only in those propositions that do not contradict with any available argument.
\end{enumerate}

Given a qualitative evidence frame $(S, \calE)$ and a proposition $P$, we say that $T$ is justification for $P$ w.r.t.\  $\calJ{}$ if $T\subseteq P$ and $T\in \calJ{}$. 




We illustrate these notions on our running example: 

\setcounter{exampleRemember}{\value{example}}
\setcounter{example}{0}
\begin{example}[continued]
Given the set of possible states $S=\{sp, dp, do, so, dm, sm\}$ and the basic  evidence set~$\calE = \{E_1, E_2, E_3\}$, the dense elements of $\topology{\calE}$ are those which contain the set~$\{dp,dm\}$. So $E_2 = \{dm,sm\}$, $E_3 = \{dp,sp\}$, $\{dp\}$ and $\{dm\}$ are not dense. 
Therefore, the corresponding frames of justification are: 
\begin{equation*}
\begin{split}
    \calJ{DS} = & \big\{E_1, E_2, E_3, \{dp\}, \{dm\}, \{dp,dm\},\\
    & \{sp,dp,do,dm\},\{dp, do, dm, sm\},\\
    &\{sp,dp,dm,sm\}, S\big\} = \topology{\calE}\setminus \{\emptyset\}.
\end{split}
\end{equation*}
\begin{equation*}
\begin{split}
    \calJ{SD} = &\{ E_1, \{dp,dm\}, \{sp,dp,do,dm\},\\
    &\{dp, do, dm, sm\},\{sp,dp,dm,sm\},\\
    &S\} = \topology{\calE}\setminus \{\emptyset, E_2,E_3,\{dp\}, \{dm\} \}.
\end{split}
\end{equation*}
\end{example}
\setcounter{example}{\value{exampleRemember}}

\subsection{Quantitative Layer}
This layer combines the degrees of certainty of the basic pieces of evidence via a mass function defined over the set of basic pieces of evidence $\calE$. This way we obtain certainty values for every combination of pieces of evidence and these values sum up to one.  

Let $(S, \calEQ)$ be a quantitative evidence frame such that $\calEQ=\{(E_1,p_1),\dots,(E_m,p_m)\}$. We will merge these different certainty  values by the following function $\delta: 2^{\calE} \rightarrow [0,1]$: 

\begin{equation}\label{def:delta}
    \delta(\E) = \prod_{E_i\in\E} p_i \prod_{E_i\notin \E} 1-p_i
\end{equation}

This function is a mass function defined over $\calE$ (as introduced in Definition \ref{defn:mass}). \shortversiononly{The proof of this claim follows directly from the definition of $\delta$.}

\fullversiononly{
\begin{extraproposition}\label{prop:multilayer-mass-function}
Given a quantitative evidence frame $(S, \calE^Q)$, the function $\delta$ defined equation (\ref{def:delta}) is a mass function over $\calE$.
\end{extraproposition}

\begin{proof} Given $\E \in 2^{\calE}$, the function $\delta$ is well-defined and the value $\delta(\E)$ is between $0$ and $1$ for $p_i \in (0,1)$ for all $i = 1,\dots, m$. In addition, the total sum of $\delta(\E)$ for all $\E \in 2^{\calE}$ is $1$. Let $\calE_m$ be the subset of $\calE$ formed by its first $m$ elements. Now, let us apply induction on the number of pieces of evidence in  $\calEQ$. For $m=1$, $2^{\calE_1} = \{\emptyset, \{E_1\}\}$ and
\begin{equation*}
\begin{split}
&\sum_{\E\in 2^{\calE_1} } \delta(\E) =
\prod_{E_i\in\emptyset} p_i \prod_{E_i\notin \emptyset} (1-p_i)\quad + \\[0.5em]
&\prod_{E_i\in \{E_1\}} p_i \prod_{E_i\notin \{E_1\}} (1-p_i) = 1-p_1 + p_1 =  1.
\end{split}
\end{equation*}
Let us assume that $\sum_{\E\subseteq \calE_{m-1} } \delta(\E) = 1$. Given $m$ pieces of evidence in $\calEQ$, let us consider the partition $2^{\calE_m} = 2^{\calE_{m-1}} \cup \big\{\E\cup \aybuke{\{E_m\}} | \E \in 2^{\calE_{m-1}}\big\}$ of $2^{\calE_m}$. Then, by the inductive hypothesis, the sum $\sum_{\E\in 2^{\calE_m} } \delta(\E)$ is equal to

\begin{equation*}
\begin{split}
     &\sum_{\E\in 2^{\calE_{m-1}}}\Bigl( (1-p_m)\prod_{E_i\in\E} p_i \prod_{E_j\notin \E} (1-p_j)\Bigr)  + \\
    &\sum_{\E\in 2^{\calE_{m-1}}}\Bigl( p_m\prod_{E_i\in\E} p_i \prod_{E_j\notin \E} (1-p_j)\Bigr) = \\[1em]
    &(1-p_m)\cdot 1 + p_m\cdot 1 =  1.
\end{split}
\end{equation*}%

\end{proof}
}

Intuitively, the function $\delta$ distributes the degree of certainty of a piece of evidence $E \in \calE$ among all the possible occurrences of $E$ in presence of other pieces of evidence. To illustrate, consider a quantitative evidence frame $(S,\calE^Q)$ where $\calE^Q  = \{(E_1, p_1), (E_2, p_2), (E_3, p_3)\}$. We then understand $\delta(E_1,E_2) = p_1p_2(1-p_3)$  as the induced certainty of the event `at least $E_1$ and $E_2$ are true' in the context of having exactly three pieces of evidence. This intuition is supported by the fact that, given a quantitative evidence frame $(S, \calEQ)$ and $\delta$ defined  as in equation (\ref{def:delta}), we have, for every $E_i \in \calE$, that 
\begin{equation*}
    \sum\limits_{\substack{\E\subseteq \calE:\\ E_i \in \E }}\delta(\E) = p_i.
\end{equation*} 
\fullversiononly{
\begin{extraproposition}\label{prop:meaning-delta}
Given a quantitative evidence frame $(S, \calE^Q)$ and $\delta$ defined  as in equation (\ref{def:delta}), we have, for every $E_i \in \calE$, that
\begin{equation*}
\sum_{\substack{\E\subseteq \calE:\\ E_i \in \E }}\delta(\E) = p_i.
\end{equation*}
\end{extraproposition}

\begin{proof}
 Without loss of generality, let us prove the result for $E_i = E_1$.

\begin{equation}\label{eq:meaning-delta}
\begin{split}
&\sum_{\substack{\E\subseteq \calE:\\ E_1 \in \E }}\delta(\E) =
\sum_{\substack{\E\subseteq \calE:\\ \aybuke{E_1} \in \E}}\Bigl( p_1 \prod_{\substack{E_i\in\E:\\ E_i\neq E_1}} p_i \prod_{\substack{E_i\notin \E:\\ E_i\neq E_1}} (1-p_i)\Bigr) = \\
&p_1\cdot\sum_{\E\subseteq \calE\setminus \{E_1\}} (\prod_{E_i\in\E} p_i \prod_{\substack{E_i\notin \E,\\ E_i\neq E_1}} 1-p_i\Bigr)
\end{split}
\end{equation}

Defining the set of evidence $\calE^\prime = \calE\setminus \{E_1\}$,  we have that
\begin{equation*}
\sum_{\E\subseteq \calE^\prime} \prod_{E_i\in\E} p_i \prod_{\substack{E_i\notin \E,\\ E_i\neq E_1}} 1-p_i = \sum_{\E\subseteq \calE^\prime} \delta(\E),
\end{equation*}
where $\delta$ is defined over $\calE^\prime$ now. By Proposition \ref{prop:multilayer-mass-function}, $\sum_{\E\subseteq \calE^\prime} \delta(\E) = 1$ and equation (\ref{eq:meaning-delta}) is equal to $p_1$.
\end{proof}
}%
\shortversiononly{(The proof follows immediately from the definition of $\delta$.)} Nevertheless, it is easier to read $\delta$ as a {\em system} of certainty assignments rather than trying to interpret each $\delta(\E)$ epistemically, since it assigns values to every combination of pieces of evidence. It is merely a method to distribute the certainty of a single piece of evidence to all the ways it can be observed in combination with the other evidence pieces. 

\setcounter{exampleRemember}{\value{example}}
\setcounter{example}{0}
\begin{example}[continued]
Given the quantitative evidence set ${\calE}^Q = \{(E_1, 0.9),(E_2, 0.75), (E_3, 0.45)\}$, the power set of $\calE$ is  $2^{\calE} = \{\emptyset, \{E_1\}, \{E_2\}, \{E_3\},$
$ \{E_1,E_2\}, \{E_1,E_2\}, \{E_2,E_3\}, \{E_1,E_2,E_3\}\}$ and the function $\delta$ returns the values presented in Table \ref{tab:delta-values}.

\begin{table}
\centering
\begin{tabular}{p{25pt}%
p{5pt}%
p{30pt}%
p{57pt}%
p{5pt}%
p{10pt}}
$\delta(\{\emptyset\})$ &$=$& $0.01$
&
$\delta(\{E_1,E_2\})$ &$=$& $0.37$   
\\
$\delta(\{E_1\})$ &$=$& $0.12$
& 
$\delta(\{E_1,E_3\})$ &$=$& $0.10$ 
\\
$\delta(\{E_2\})$ &$=$& $0.04$ 
& 
$\delta(\{E_2,E_3\})$ &$=$& $0.03$
\\
$\delta(\{E_3\})$ &$=$& $0.01$
&
$\delta(\{E_1,E_2,E_3\})$ &$=$& $0.30$
\\
\end{tabular}
\caption{Image of the function $\delta$ in the running example.}
\label{tab:delta-values}
\end{table}

\end{example}
\setcounter{example}{\value{exampleRemember}}

\subsection{Bridging Layer}
Having introduced the qualitative and quantitative layers,  we can now connect these two to calculate degrees of beliefs based on a quantitative evidence model $(S, \calEQ)$ and a given frame of justification $\calJ{}$. To this end, we start the section by defining a family of functions to map $2^{\calE}$ - 
the domain of the mass function $\delta$ defined in equation (\ref{def:delta}) - 
to $\topology{\calE}$. %
Finally, we will define both a mass function and a basic probability assignment over $\topology{\calE}$ that will be used to compute the degrees of belief this model returns. 

\subsubsection{Evidence Allocation Functions}

At this point, we have a set of justifications $\calJ{}$ and a merged measure of certainty distributed over the elements of $2^{\calE}$. Our goal is to link the mass values defined for the elements of $2^{\calE}$ by $\delta$ to the elements of $\topology{\calE}$ and, in turn, to the elements of $\calJ{}$. %
Mapping these two sets is not a trivial issue as there are many ways to do so. For example, given two pieces  of evidence $E_1$ and $E_2$ in $\calE$, the $\delta$-value associated with the set $\{E_1,E_2\}$ could be mapped to different elements of $\topology{\calE}$ depending on the interpretation we want to give to these values. A strict interpretation could state that every piece of evidence in the set $\{E_1,E_2\}$ contains the actual world, which would map $\{E_1,E_2\}$ to $E_1\cap E_2$. Conversely, a moderate interpretation could state that  at least one of the elements of $\{E_1,E_2\}$ contains the actual world,  which would map $\{E_1,E_2\}$ to $E_1\cup E_2$. We capture various ways of interpreting the mass values provided by $\delta$ via {\em evidence allocation functions}.

\begin{definition}\label{def:allocation_funtion}
    Let $(S, \calE)$ be a qualitative evidence frame. A \emph{set of evidence allocation functions} $\mathfrak{F}$ on $(S, \calE)$ is a set of of functions from $2^{\calE}$ to $\topology{\calE}$ (the topology generated by $\calE$) such that for all $f, g\in \mathfrak{F}$: 
    \begin{enumerate}
    \item\label{def:allocation_funtion.1} $f(\emptyset) = S$,
    \item\label{def:allocation_funtion.2} for all \aybuke{nonempty} $\E\subseteq \calE$, $f(\E)\in\topology{\E}$ (the topology generated by $\E$) and it is dense in $\medcup\E$ w.r.t.\  $\topology{\E}$; or $f(\E)= \emptyset$.\footnote{Note that since $\E\subseteq \calE$, we have $\topology{\E}\subseteq \topology{\calE}$ (follows immediately by the definitions of generated topologies).}
    \item\label{def:allocation_funtion.3} for all $\E\subseteq \calE$ and every $f$, $g$ in $\mathfrak{F}$, $f(\E)\subseteq g(\E)$ or $g(\E)\subseteq f(\E)$.
    \end{enumerate}
\end{definition}

In what follows, we assume that all evidence allocation functions are defined on a qualitative evidence frame $(S, \calE)$ and omit mention of it. 

The first item of this definition preserves the notion of uncertainty since in our context it is modeled by associating the value $1-p_i$ (for $i=1,\dots,m$) with the total set. %
On the other hand, Conditions (\ref{def:allocation_funtion.2}) and (\ref{def:allocation_funtion.3}) establish some minimal rationality constraints. Condition (\ref{def:allocation_funtion.2}) states that an evidence allocation function $f$ assigns $\E$ to an argument that is generated by $\E$ and that does not contradict with any other argument generated by $\E$. So, an evidence allocation function does not assign a set of evidence $\E$ to some argument that cannot be produced within $\E$ or that is inconsistent with $\E$. Condition (\ref{def:allocation_funtion.3}) ensures that two agents with the same set of evidence allocation functions will associate $\E$ with arguments such that one entails the other. In this sense, $\E$ cannot pull these agents to different directions with respect to their evidence. Let us see three examples of these functions. 

\begin{proposition}\label{prop:intersection}
Given a set of evidence allocation functions $\mathfrak{F}$ and the function $i: 2^{\calE} \rightarrow \topology{\calE}$ such that 
\begin{equation*}
  i(\E) = \begin{dcases*}
    \medcap \E & if $\E \neq \emptyset$, \\
    S & \text{otherwise.} \\
  \end{dcases*}
\end{equation*}
the set $\mathfrak{F} \cup \{i\}$ is a set of  evidence allocation functions. 
\end{proposition}

\begin{proof}
First, every element of $\mathfrak{F} \cup \{i\}$ maps the empty set to $S$ by definition. Secondly, recall that the topology $\topology{\E}$ is generated by the element of $\E$ by closing it under finite intersections and arbitrary unions. This implies that $i(\E)\in \topology{\E}$ and $i(\E) = \medcap \E\subseteq T$ for all $T\in \topology{\E}{\setminus}\{\emptyset\}$ (i.e., $i(\E)$ is the smallest element of $\topology{\E}$). Consequently, $i(\E)$ is either the empty set or it is a dense element of $\topology{E}$. 
In addition, given $f \in \mathfrak{F}$, $f(\E) \in \topology{\E}$ (by Definition \ref{def:allocation_funtion}). So, if $f(\E)\not = \emptyset$, then $i(\E) \subseteq f(\E)$. 
\end{proof}

\begin{proposition}\label{propo:union}
Given a set of evidence allocation functions $\mathfrak{F}$ and the function $u: 2^{\calE} \rightarrow \topology{\calE}$ such that 
 \begin{equation*}
  u(\E) = \begin{dcases*}
    \medcup \E & if $\E \neq \emptyset$, \\
    S & otherwise. \\
  \end{dcases*}
\end{equation*}
the set $\mathfrak{F} \cup \{u\}$ is a set of evidence allocation functions. 
\end{proposition}

\shortversiononly{
\begin{proof} Proof is similar to the proof of Proposition \ref{prop:intersection} and follows from the fact that for all nonempty $\E \subseteq \calE$, $\medcup\E$ is the largest element of $\topology{\E}$.
\end{proof}
}
\fullversiononly{
\begin{proof}
First, every element of $\mathfrak{F} \cup \{u\}$ maps the empty set to $S$ by definition. Secondly, considering $\E \in 2^{\calE}$, the topology $\topology{\E}$ is a set generated by the elements of $\E$ by closing it under finite intersections and arbitrary unions. \aybuke{This implies in particular that $u(\E) = \medcup \E$ is the largest element of $\topology{\E}$. Therefore, for all $T\in \topology{\E}\setminus\{\emptyset\}$, $T\cap  u(\E)=T\not = \emptyset$, meaning that  $u(\E)$ is also dense in $\medcup \E$ w.r.t. $\topology{\E}$.}
 In addition, given $f \in \mathfrak{F}$, $f(\E)$ belongs to $\topology{\E}$ (by Definition \ref{def:allocation_funtion}). So $f(\E) \subseteq u(\E)$.
\end{proof}
}

We then obtain the following corollary, showing the boundaries of possible evidence allocation functions:


\begin{corollary}
Given a set of evidence allocation functions $\mathfrak{F}$, a function $f\in \mathfrak{F}$, and $\E\in 2^{\calE}$, we have $i(\E)\subseteq f(\E)\subseteq u(\E)$.
\end{corollary}
\begin{proof}
It is easy to see that if $u(\E)\subset f(\E)$ or $f(\E)\subset i(\E)$, $f$ violates Definition \ref{def:allocation_funtion}.\ref{def:allocation_funtion.2}.
\end{proof}

For our last example, we need the following auxiliary lemma. 

\begin{lemma}\label{lemma:minimum_dense_set}
Let $(S, \calE)$ be a qualitative evidence frame, $\E \in 2^{\calE}$, and $\text{dense}(\E)$ be the set of dense elements of $\topology{\E}$ in $\medcup{\E}$. Then the order $(\text{dense}(\E),\subseteq)$ has a minimum.
\end{lemma}

\begin{proof} Let $\mathbb{M}$ be the set of all $\A\in 2^{\E}$ such that (1) $\medcap \A\not = \emptyset$, and (2) if $\A'\in 2^{\E}$ and $\A\subset \A'$, then $\medcap \A'=\emptyset$. In other words, $\mathbb{M}$ is the set of all subsets of  $2^{\E}$ that satisfy the maximal finite intersection property. Then 
\begin{equation*}
    \text{min}\big((\text{dense}(\E),\subseteq)\big)= \medcupeq \{\medcap \A | \A \in \mathbb{M}\}.
\end{equation*}
\fullversiononly{
To simplify notation, let $M:= \medcupeq \{\medcap \A | \A \in \mathbb{M}\}$. First, $M$ is in $\topology{\E}$ \aybuke{by the definition of generated topology} since every $\A \in \mathbb{M}$ only contains elements of $\E$. In addition, $M$ is dense \aybuke{w.r.t. $\topology{\E}$}. Any element $T \in \topology{\E}$ forms part of a maximal set with nonempty intersection: if $T$ has empty intersection with any other element of $\topology{\E}$, then $\{T\}$ is a maximal set with nonempty intersection. Since $M$ contains the intersections of all these maximal sets, in particular, $M \cap T \neq \emptyset$.

Now, let us prove that $M\subseteq T$ for every element $T \in \text{dense}(\E)$. Let $\A$ be an arbitrary element of $\mathbb{M}$ and $T$ an arbitrary element of $\text{dense}(\E)$. Since $\A$ is a subset of $\E$, $\medcap \A \in \topology{\E}$. In addition, $T$ being dense in $\topology{\E}$ implies that $T \cap \medcap\A \neq \emptyset$. Let us take the subset of $\E$ formed by every element of $\E$ which contains $T\cap \medcap \A$. This set is nonempty since it includes at least $\A$. If this set contains a piece of evidence $E \in \E$ such that $E \not\in \A$, then $\medcap \A \cap \{E\}$ is nonempty and $\A$ does not satisfy the maximal finite intersection property. Therefore, every piece of evidence $E$ in $\E$ that contains $T\cap \medcap\A$ is in $\A$. By definition of topology generated by $\E$, if every element in $\E$ that contains $a$ also contains $b$, then every element of $\topology{\E}$ that contains $a$ contains $b$ too. In this particular case, we have shown that every element of $\E$ that contains $T\cap \medcap \A$ also contains $\medcap \A$. So every element of $\topology{E}$ that contains $T\cap\medcap\A$, for example $T$, also contains $\medcap \A$. Hence $M\subseteq T$ for any $T \in \text{dense}(\E)$, that is, $M$ is the minimum of the dense sets in $\topology{\E}$ with the subset relation.}
\end{proof}

\begin{proposition}\label{prop:smallest-dense-set}
Given a set of evidence allocation functions $\mathfrak{F}$ and the function $d: 2^{\calE} \rightarrow \topology{\calE}$ such that 
\begin{equation*}
  d(\E) = \begin{dcases*}
    \text{min}\big((\text{dense}(\E),\subseteq\big)) & if $\E \neq \emptyset$, \\
    S & otherwise. \\
  \end{dcases*}
\end{equation*}
the set $\mathfrak{F} \cup \{d\}$ is a set of evidence allocation functions.
\end{proposition}
\begin{proof}
The first two conditions of Definition \ref{def:allocation_funtion} hold by definition. In addition, for any $f \in \mathfrak{F}$, we know $f(\E)=\emptyset$ or $f(\E)$ is dense w.r.t. $\topology{\E}$. In the former case, $f(\E) \subseteq d(\E)$. In the latter one, $d(\E) \subseteq f(\E)$ by Lemma \ref{lemma:minimum_dense_set}.
\end{proof}

\begin{corollary}
    The set of functions $\mathfrak{F}= \{i,u,d\}$ is a set of evidence allocation functions.
\end{corollary}

These evidence allocation functions can be understood as contextual parameters. Depending on the context, agents may prefer to place the available information in weaker or stronger arguments (arguments which contain more or fewer elements respectively), to avoid increasing the uncertainty of the model (as it can happen by mapping elements of $2^{\calE}$ to the total set $S$), or to avoid discarding information (as it can happen by mapping elements of $2^{\calE}$ to the empty set). 


\setcounter{exampleRemember}{\value{example}}
\setcounter{example}{0}
\begin{example}[continued]
   Given our quantitative evidence frame $(S, \calEQ)$ with $S = \{sp,dp,do, so,dm,sm\}$, $\calEQ = \{(E_1, 0.9),$
$ (E_2, 0.75), (E_3, 0.45)\}$, and  the set of evidence allocation functions $\bm{\mathfrak{F}}=\{i,u,d\}$, we obtain the results presented in Table \ref{tab:evidence-allocation-function}. As we can see, all these maps produce the same images for singletons. In addition, map $d$ sometimes returns the same result as  $i$ (rows 5,6), sometimes it returns the same result as $u$ (row 7) and sometimes it returns a value between the ones given by $i$ and  $u$ (row 8). 
    \begin{table}
    \centering
    \begin{tabular}{m{.15\linewidth}C{.10\linewidth}C{.20\linewidth}C{.20\linewidth}R{.10\linewidth}}
      \toprule 
      \bfseries $\E$ & \bfseries $\bm{i(\E)}$ & \bfseries $\bm{u(\E)}$ & \bfseries $\bm{d(\E)}$ & \bfseries $\bm{\delta(\E)}$ \\
      \midrule 
      $\emptyset$ & $S$ & $S$ & $S$ & $0.01$ \\[1em]
    $E_1$ & $E_1$ & $E_1$ & $E_1$ & $0.12$ \\[1em] 
    $E_2$ & $E_2$ & $E_2$ & $E_2$ & $0.04$ \\[1em]
    $E_3$ & $E_3$ & $E_3$ & $E_3$ & $0.01$ \\[1em]
    $E_1,E_2$ & $\{dm\}$ & $\{dp,do,$ $dm,sm\}$ & $\{dm\}$ & $0.37$ \\[1em]
    $E_1,E_3$ & $\{dp\}$ & $\{sp,dp,$ $do,dm\}$ & $\{dp\}$ & $0.10$ \\[1em]
    $E_2,E_3$ & $\emptyset$ & $\{sp,dp,$ $dm,sm\}$ & $\{sp,dp,$ $ dm,sm\}$ & $0.03$ \\[1em]
    $E_1,E_2,E_3$ & $\emptyset$ & $S$ & $\{dp,dm\}$ & $0.30$ \\
      \bottomrule 
    \end{tabular}
    \caption{Images by the evidence allocation functions $\{i,u,d\}$.}\label{tab:evidence-allocation-function}
\end{table}
\end{example}
\setcounter{example}{\value{exampleRemember}}

\subsubsection{Belief Functions}
Now that we know how to link the power set of $\calE$ and $\topology{\calE}$, let us define some functions that will allow us to compute degrees of belief given a body of possible mutually inconsistent and uncertain evidence as well as an agent's evidential demands.

Let $(S,\calEQ)$ be a quantitative evidence frame and $\mathfrak{F}$ a set of evidence allocation functions. Then, we define the function $\delta_{\topology{}}: \mathfrak{F}\times2^S \rightarrow [0,1]$ as:

\begin{equation*}\label{eq:delta-tau}
  \delta_{\topology{}}(f,T) = \begin{dcases*}
      \sum_{\E:f(\E)= T}\delta(\E) & if $T \in \topology{\calE}$, \\
    0 & otherwise. \\
  \end{dcases*}
\end{equation*}

Fixing $f\in\mathfrak{F}$, $\delta_{\topology{}}(f,\cdot)$ is a mass function over $S$ since $\delta$ is a mass function over $\calE$.

By adding a frame of justification $\calJ{}$ to the previous setting, we define the function $\delta_{\calJ{}}: \mathfrak{F}\times2^S \rightarrow [0,1]$ as:

\begin{equation*}\label{eq:delta-tau}
  \delta_{\calJ{}}(f,A) = \begin{dcases*}
    \frac{\delta_{\topology{}}(f,A)}{\sum_{T \in \calJ{}} \delta_{\topology{}}(f,T)} & if $A \in \calJ{}$, \\
    0 & otherwise. \\
  \end{dcases*}
\end{equation*}

Fixing $f\in\mathfrak{F}$, $\delta_{\calJ{}}(f,\cdot)$ is a basic probability assignment. %
\shortversiononly{
 The above fraction is well-defined for every non-dogmatic evidence set - %
that is, for $\calEQ=\{(E_i,p_i)\}_{i=1,\dots,m}$ such that $p_i\neq 1$ for all $i$ - %
since the set of possible states $S$ belongs to every frame of justification. In addition, $\delta_{\calJ{}}(f,\emptyset) = 0$ since $\emptyset \not\in \calJ{}$. Finally, the total sum of its values is equal to $1$ because $\delta_{\topology{}}(f,\cdot)$ is a mass function over $S$.
}

\fullversiononly{
\begin{extraproposition}\label{prop:delta_bpa}
Given a quantitative evidence frame $(S, \calEQ)$, a frame of justification $\calJ{}$, and an evidence allocation function $f \in \mathfrak{F}$, the function $\delta_{\calJ{}}(f,\cdot)$ defined as before is a basic probability assignment. 
\end{extraproposition}

\begin{proof}
The above fraction is well-defined for every non-dogmatic evidence set - %
that is, for $\calEQ=\{(E_i,p_i)\}_{i=1,\dots,m}$ such that $p_i\neq 1$ for all $i$ - %
since the set of possible states $S$ belongs to every frame of justification. In addition, $\delta_{\calJ{}}(f,\emptyset) = 0$ since $\emptyset \not\in \calJ{}$ for the definition of argument. Finally, the total sum of its values is equal to $1$ because $\delta_{\topology{}}(f,\cdot)$ is a mass function over $S$.
\end{proof}
}

At last, we can define the {\em degree of belief} for  proposition $P \subseteq S$ via a multi-layer belief function given the quantitative evidence set, the frame of justification, and the evidence allocation functions:

\begin{definition}[Multi-layer belief function]\label{def:degree-of-belief}
Let $(S,\calEQ)$ be a quantitative evidence frame, $\calJ{}$ a frame of justification and $\mathfrak{F}$ a set of evidence allocation functions. We call a function $\Bel_{\calJ{}}:\mathfrak{F}\times2^S \rightarrow [0,1]$ defined as: 
\begin{equation*}
    \Bel_{\calJ{}}(f,P) = \sum_{ A \subseteq P}\delta_{\calJ{}}(f,A).
\end{equation*}
a {\em multi-layer belief function}. When $\calJ{}$ and $f$ are clear from context, we omit mention of them and write $\Bel(P)$.
\end{definition}

\begin{proposition}
Given  a quantitative evidence frame $(S,\calEQ)$, a frame of justification $\calJ{}$, and an evidence allocation function $f$, the function $\Bel_{\calJ{}}(f,\cdot)$ defined in Definition \ref{def:degree-of-belief} is a belief function.
\end{proposition}

\begin{proof}
Since $\delta_{\calJ{}}$ is a basic probability assignment \fullversiononly{for Proposition \ref{prop:delta_bpa}}, the function $ \Bel_{\calJ{}}(f,\cdot)$ is a belief function \cite[p.~51]{SHAFER1976}.
\end{proof}

\setcounter{exampleRemember}{\value{example}}
\setcounter{example}{0}
\begin{example}[continued]
In deciding whether to stop, the degrees of beliefs of the cars in the propositions that the object is dynamic and that it is a pedestrian are critical.  Car A cannot dismiss any evidence that supports that the object is a pedestrian, so it will choose the frame of justification DS. Car B must be very sure about the outcome, so it will choose the frame of justification SD to avoid contradictory pieces of evidence. In Table \ref{tab:belief-carA} and Table \ref{tab:belief-carB} we can see the results that they would get with the information provided in the previous examples. 

\begin{table}[h!]
    \centering
    \begin{tabular}{m{.15\linewidth}C{.25\linewidth}R{.1\linewidth}R{.1\linewidth}R{.1\linewidth}}
      \toprule 
      \bfseries  Proposition & $\bm{P}$ & $\bm{i}$ & $\bm{u}$ & $\bm{d}$\\
      \midrule 
      $(1)$& $\{dp,do,$ $dm\}$ & $0.88$ & $0.59$ & $0.89$ \\   
    $(2)$ & $\{sp,dp\}$ & $0.16$ & $0.11$ & $0.11$ \\
    \midrule
    Uncertainty& $S$ & $0.02$  & $0.31$  & $0.01$ \\
    N.f.\footnotemark & $\calJ{DS}$ & $0.66$ & $1$ & $1$\\
      \bottomrule 
    \end{tabular}
    \caption{Results of Car A after computing $\Bel_{\calJ{DS}}(f,P)$ for $f\in \{i,u,d\}$ in proposition $(1)$  `the object is dynamic' and $(2)$ `the object is a pedestrian'.}\label{tab:belief-carA}
\end{table}

\begin{table}[h!]
    \centering
    \begin{tabular}{m{.15\linewidth}C{.25\linewidth}R{.1\linewidth}R{.1\linewidth}R{.1\linewidth}}
      \toprule 
      \bfseries Proposition & $\bm{P}$ & $\bm{i}$ & $\bm{u}$ & $\bm{d}$\\
      \midrule 
      $(1)$& $\{dp,do,$ $dm\}$ & $0.92$ & $0.13$ & $0.91$ \\   
    $(2)$ & $\{sp,dp\}$ & $0$ & $0$ & $0$ \\
    \midrule
     Uncertainty & $S$ & $0.08$  & $0.33$  & $0.0.2$ \\
    N.f.\footnotemark[\value{footnote}] & $\calJ{SD}$ & $0.13$ & $0.93$ & $0.46$\\
      \bottomrule 
    \end{tabular}
    \caption{Results of Car B after computing $\Bel_{\calJ{SD}}(f,P)$ for $f\in \{i,u,d\}$ in proposition $(1)$ `the object is dynamic' and $(2)$ `the object is a pedestrian'.}\label{tab:belief-carB}
\end{table}

\footnotetext{Normalization factor.}

Although the ultimate decision will depend on the decision-making process that is using these degrees of belief, both cars obtain high degrees of belief ($>0.85$) in proposition $(1)$ (Table \ref{tab:belief-carA}) for the evidence allocation functions $i$ and $d$. They obtain a lower degree for evidence allocation function $u$, but a higher level of uncertainty. These results are coherent with the piece of evidence `It is a dynamic object with $90\%$ certainty'. In addition, car A obtains some degree of belief also in proposition $(2)$ (Table \ref{tab:belief-carA}), giving the system the chance to decide whether to stop or not. However, car B rejects the option of having a pedestrian in front of it: every evidence allocation function returns a null degree of belief in this case. Both results are coherent with the collected data: there is some certainty about  `It is a pedestrian' ($45\%$) but higher certainty about `It is a motorbike' ($75\%$).

\end{example}
\setcounter{example}{\value{exampleRemember}}

\section{Assessing the Multi-Layer Model}\label{sec:assessment}

In the previous section we defined the multi-layer belief model and showed that it is able to process bodies of possible mutually inconsistent and uncertain evidence while taking into account the agent's evidential demands. This outcome has been obtained by merging tools of DST,  %
such as {\em belief function} and {\em degree of uncertainty}, %
and TME, %
such as {\em denseness} and {\em justification}. 

In this section, we show that the multi-layer belief model is able to reproduce both Dempster's rule of combination and the topological models' outcome. In addition, we will check some basic computational properties to know the prior strengths and limitations of the model from a practical point of view. 

\subsection{Relationship with other belief models}\label{subsec:comparison}

Let $(S,\calE^Q)$ be a quantitative evidence frame and $k$ the number of elements in $\calE^Q$. When we work with the Dempster-Shafer frame of justification and the function $i$ as evidence allocation function, we obtain exactly the belief function given by applying the Dempster's rule of combination to the basic probability assignments $\{m_i| i = 1,\dots,k\}$ defined by $m_i(E_i) = p_i$, $m_i(S) = 1-p_i$ for every $i = 1,\dots,k$. 

\begin{proposition}\label{prop:generalization-dst}
Let us take $(S,\calE^Q)$ and $k$ as above, the evidence allocation function $i: 2^{\calE} \rightarrow \topology{\calE}$ as defined in Proposition \ref{prop:intersection}, and the frame of justification $\calJ{DS}$. Let us consider the belief function  $\Bel_{\calJ{DS}}(i,\cdot): 2^S \to [0,1]$ defined by the multi-layer belief model and $\Bel: 2^S \to [0,1]$ a belief function obtained by applying the Dempster's rule of combination to the basic probability assignments $\{m_i| i = 1,\dots,k\}$ such that $m_i(E_i) = p_i$, $m_i(S) = 1-p_i$ for every $i = 1,\dots,k$. Then, 
\begin{equation*}
\Bel_{\calJ{DS}}(i,P) = \Bel(P)
\end{equation*} for every $P \subseteq S$.
\end{proposition}

\fullversiononly{
\begin{proof}
Let us consider $m = \oplus_i m_i$. Given a proposition $P \subseteq S$, $\Bel_{\calJ{DS}}(i,P) = \sum_{A \subseteq P} \delta_{\calJ{DS}}(i,A)$ and $\Bel(P)= \sum_{A\subseteq P} m(A)$. So let us prove that $\delta_{\calJ{DS}}(i,A) = m(A)$ to prove the claim. 

On the one hand, for $A \in \calJ{DS}$
\begin{equation}\label{eq:delta_hash}
   \delta_{\calJ{DS}}(i,A) = \frac{\delta_{\topology{}}(i,A)}{\sum_{T\in \calJ{DS}}\delta_{\topology{}}(i,T)} 
\end{equation}

where $\calJ{DS} = \topology{\calE}\setminus \{\emptyset\}$ and $\delta_{\topology{}}(i,A) = \sum_{\medcapeq\E = A} \delta(\E)$ with $\E \in 2^{\calE}$. On the other hand, 
\begin{equation}\label{eq:dempster_rule}
   m(A) = \frac{\sum_{\medcapeq_i A_i = A}m_1(A_1)\dots m_k(A_k)}{\sum_{\medcapeq_i B_i\neq \emptyset}m_1(B_1)\dots m_k(B_k)}. 
\end{equation}

Let us see  that 

\begin{equation*}
\sum_{\medcapeq_i A_i = A}m_1(A_1)\dots m_k(A_k) = \sum_{\medcapeq \E = A} \delta(\E).
\end{equation*}

The functions $m_i$  are basic probability assignments such that $m_i(A)\neq 0$ only when $A = E_i$ or $A= S$. Therefore, 

\begin{equation*}
\resizebox{\linewidth}{!}{$
    \displaystyle
\sum_{\medcapeq_i A_i = A}m_1(A_1)\dots m_k(A_k)\quad =
\sum_{\substack{A_i\in\{E_i,S\}:\\\medcapeq_i A_i = A}}\!m_1(A_1)\dots m_k(A_k).$}
\end{equation*}

Let us define a family of vectors $\bar{b}$ of length $k$ such that $b_i \in \{E_i,S\}$. Then, the previous formula is equal to 

\begin{equation*}
\sum_{\bar{b}:\medcapeq_i b_i = A}m_1(b_1)\dots m_k(b_k).
\end{equation*}

There is a one-to-one correspondence between the family of vectors $\bar{b}$ and $2^{\calE}$. For example, we can define the function $f: \{E_i,S\}^k\to 2^{\calE}$ such that $f(\bar{b}) = \{E_i| b_i \neq S\}$. This function is bijective and $\medcap b_i = A$ if and only if $\medcap f(\bar{b}) = A$. This means that for every vector $\bar{b}$ such that $\medcap b_i =A$ there exists exactly one $\E\in 2^{\calE}$ such that $\medcap \E = A$.  Consequently, the numerator of equation (\ref{eq:dempster_rule}) is equal to 

\begin{equation*}
\sum_{\medcapeq \E = A} \Bigl(\prod_{E_i\in \E}m_i(E_i)\prod_{E_j\notin \E}m_j(S)\Bigr)
\end{equation*}

which is equal to 

\begin{equation*}
\sum_{\medcapeq \E = A} \Bigl(\prod_{E_i\in \E}p_i\prod_{E_j\notin \E}(1-p_j)\Bigr) = \sum_{\medcapeq \E = A} \delta(\E).
\end{equation*}

We also know that every $B_i$ in the denominator of equation (\ref{eq:dempster_rule}) belongs to $\{E_i,S\}$, so we can consider the same family of vectors $\bar{b}$ and the one-to-one correspondence between this family and $2^{\calE}$. If we unfold the definition of the denominator of equation (\ref{eq:delta_hash}), we get: 

\begin{equation*}    
\sum_{T\in \calJ{DS}}\delta_{\topology{}}(i,T) = \sum_{T\in\topology{\calE}\setminus \{\emptyset\}}\Bigl(\sum_{\medcapeq \E=T}\delta(\E)\Bigr)
\end{equation*}

which is equal to 

\begin{equation*}
\sum_{\medcapeq \E\neq \emptyset}\Bigl(\prod_{E_i\in \E}p_i\prod_{E_j\notin \E}(1-p_j)\Bigr)
\end{equation*}

By applying the bijective function $f$ that was defined above, this formula is equal to 

\begin{equation*}
\resizebox{\linewidth}{!}{$
    \displaystyle
\sum_{\bar{b}:\medcapeq b_i \neq \emptyset}m_1(b_1)\dots m_k(b_k)\ = \sum_{\medcapeq_i B_i\neq \emptyset}m_1(B_1)\dots m_k(B_k).$}
\end{equation*}
\end{proof}
}

As we mentioned in the introduction of this paper, the classical version of Dempster's combination rule has led to some counter-intuitive results when used in highly conflicting contexts \cite{ZADEH1986}. Many authors have justified the use of variations of this rule to avoid this issue~\cite{LEFEVRE2002,SMETS2007,PICHON2010}. Two characteristic examples are Yager’s~\cite{Yager87} and Dubois-Prade’s rules~\cite{DUBOIS1992}. The former allows the empty set to have a non-negative value in the mass function. This value will be added to the degree of belief of the total set, so normalization is no longer necessary. The second example proposes to combine two mass functions by considering the union when DRC uses the intersection, so normalization is not necessary either. The multi-layer model can capture the intention of both approaches:  to allocate certainty values in a similar way to Yager's rule, we can consider an evidence allocation function such that it assigns the nonempty elements $E$ of $2^{{\calE}}$ to the intersection if it is nonempty, and to the total otherwise. To proceed similarly to Dubois-Prade’s rule it is enough to consider the union as evidence allocation function. The results obtained by doing this are not exactly the same: considering a quantitative evidence set $\calE^Q$ with three pieces of evidence, the multi-layer belief model would associated the value $p_1\cdot p_2\cdot (1-p_3)$ to sets which receive the value $p_1\cdot p_2$ with the other rules. However, \aybuke{our framework also} avoids normalization: $\calJ{DS} = \topology{\calE}\setminus\{\emptyset\}$ and $\delta_{\topology{}}(\emptyset)= 0$ in both cases.

Finally, the multi-layer belief model can also return the same outcome as TME. In this case, we must consider the strong denseness frame of justification and the minimum dense set allocation function $d$. Notice that TME is a qualitative approach, so the only properties \aybuke{we will use of the $p_i$ values are that they are positive and smaller than $1$.}


\begin{proposition}\label{prop:generalization-te}
Given a quantitative evidence frame $(S,\calE^Q)$ where  $\calEQ = \{(E_i,p_i)\}_{i=1,\dots,m}$ such that $p_i \in (0,1)$ for all $i$, and the evidence allocation function $d: 2^{\calE} \rightarrow \topology{\calE}$ defined in Proposition \ref{prop:smallest-dense-set},  let us consider  $\Bel_{\calJ{SD}}: 2^S \to [0,1]$ the belief function defined by the multi-layer belief model; and $\text{B}: 2^S \to \{0,1\}$ a belief operator such that $\text{B}(P)$ if and only if there exists $D \in \topology{\calE}$ such that  $D\subseteq P$ and $D \cap T \neq \emptyset$ for all $T \in \topology{\calE}\setminus\{\emptyset\}$. Then, 
\begin{equation*}
\text{B}(P) = 1 \text{ if and only if\ }\ \Bel_{\calJ{SD}}(d,P) > 0 
\end{equation*} for every $P \subseteq S$.   
\end{proposition}

\begin{proof}
Let us prove the left to right direction. By definition,  $\text{B}(P) = 1$ if and only if there exists $D \in \topology{\calE}$ such that  $D\subseteq P$ and $D \cap T \neq \emptyset$ for all $T \in \topology{\calE}\setminus \{\emptyset\}$. That is if  $\text{B}(P) = 1$ then there is a set $D$ contained in $P$ which is in $\topology{\calE}$ and is dense in it. Therefore,  $d(\calE) \aybuke{\subseteq} D$ since the topology generated by $\calE$ is exactly $\topology{\calE}$ and the proof of Lemma \ref{lemma:minimum_dense_set} shows that $d({\aybuke{\calE}})$ is included in every element of $\text{dense}({\aybuke{\calE}})$. Consequently, 

\begin{equation*}
\Bel_{\calJ{SD}}(d,P) = \sum_{A\subseteq P}\delta_{\calJ{SD}}(d,A) \geq \delta_{\calJ{SD}}(d,d(\calE))
\end{equation*}

And $\delta_{\calJ{SD}}(d,d(\calE)) > 0$ if and only if $d(\calE) \in \calJ{SD}$ and $\delta_{\topology{}}(d,d(\calE)) > 0$. Both conditions hold by definition:  $\delta_{\topology{}}(d,d(\calE)) = \sum_{d(\E)=d(\calE)} \delta(\E)$ which is greater or equal to $\delta(\calE) = p_1\dots p_k > 0$.

Now, let us see the right-left implication. If $\Bel_{\calJ{SD}}(d,P) > 0$ then there is a set $\E\in 2^{\calE}$ such that $d(\E) \subseteq P$ and $d(\E) \in \calJ{SD}$. This means that $d(\E)$ is a dense element of $\topology{\calE}$, so there exists $D \in \topology{\calE}$ such that  $D\subseteq P$ and $D \cap T \neq \emptyset$ for all $T \in \aybuke{\topology{\calE}\setminus\{\emptyset\}}$. This proves that $B(P) = 1$.
\end{proof}

\subsection{Computational Complexity}\label{subsec:complexity}

In this section, we will provide a basic analysis of the computational
complexity of computing degrees of belief using the model that we proposed
in Section~\ref{sec:model}.
In particular, we will describe in precise terms what computational problem
we consider, and we provide some first computational complexity results.
We will assume familiarity with the theory of computational complexity - %
and in particular with the complexity class \#P.
(For details on this, we refer to textbooks on the topic,
e.g.,~\cite{AroraBarak09}).

\paragraph{Computational problem}
We consider the following computational problem.
The input consists of a set~$S$ of possible states,
a quantitative evidence
set~$\calEQ = \{(E_1,p_1),\dots,(E_m,p_m)\} \subseteq 2^S \times (0,1)$ - %
where we use~$\calE$ to denote~$\{ E_1,\dots,E_m \}$ - %
a frame of justification~$\calJ{}$,
an evidence allocation
function~$f : 2^{\calE} \rightarrow \topology{\calE}$,
and a proposition~$P \in 2^{S}$.
The task is to compute the degree~$\Bel_{\calJ{}}(f,P)$ of belief
for the proposition~$P$.
For the sake of convenience, we will refer to this computational
problem as \textsc{Degree of Belief}.

$\calJ{}$ and the domain of the function~$f : 2^{\calE} \rightarrow \topology{\calE}$
is in general of size exponential in the size of~$S$
(and the other parts of the input).
Therefore, whenever we do not consider a fixed frame of justification
or a fixed evidence allocation function, respectively,
we assume that these functions are represented as (a suitable specification of)
a polynomial-time computable function.

\paragraph{Upper bound}
We begin with an upper bound on the computational complexity
of the problem in its most general form - that is,
when the frame of justification and evidence allocation function are given
as part of the input.

\begin{proposition}
\label{prop:sharp-p-membership}
\textsc{Degree of Belief} is in \#P,
if the frame of justification~$\calJ{}$
has a polynomial-time decidable characteristic function
and if the evidence allocation function~$f$ is polynomial-time computable,
and both of these are given as part of the input
(specified in a suitable format).
\end{proposition}
\fullversiononly{
\begin{proof}
We will show that (a suitable variant of) \textsc{Degree of Belief} is in \#P.
\textsc{Degree of Belief} returns fractions~$q \in \mathbb{Q}$,
and the complexity class \#P concerns functions that return natural numbers.
Therefore, in the remainder of this proof,
we specify a fraction~$q \in \mathbb{Q}$
by two natural numbers~$n,d \in \mathbb{N}$ such that~$q = \nicefrac{n}{d}$.
One can straightforwardly extend results for \#P functions that return natural numbers
to \#P functions that specify fractions in this way.
We omit further details of this in this proof.

We will show that the function~$\Bel_{\calJ{}}(f,P)$ is in \#P
in several steps.
In particular, we will show that~$\delta(\E)$,
$\delta_{\topology{}}(f,T)$ and~$\delta_{\calJ{}}(f,A)$ are computable in \#P,
using each result to establish the next result.
Then, using these intermediate results,
we will show that~$\Bel_{\calJ{}}(f,P)$ is in \#P.
In order to do this, we will use various closure properties of \#P
\cite{OgiwaraHemachandra93}.
These closure properties can, as mentioned above,
be straightforwardly extended to \#P functions that return fractions~$q = \nicefrac{n}{d}$
(by specifying~$n$ and~$d$).

Take an input consisting of a set~$S$ of possible states,
a quantitative evidence
set~$\calEQ = \{(E_1,p_1),\dots,(E_m,p_m)\} \subseteq 2^S \times (0,1)$ - %
where~$\calE$ denotes~$\{ E_1,\dots,E_m \}$ - %
a frame~$\calJ{}$ of justification,
an evidence allocation function~$f : 2^{\calE} \rightarrow \topology{\calE}$ - %
where the frame of justification and the evidence allocation function
are both given by suitably specified polynomial-time computed functions - %
and a proposition~$P \in 2^{S}$.

Firstly, we show that~$\delta(\E)$ is in \#P.
Remember that $\delta(\E) = \prod_{E_i\in\E} p_i \prod_{E_i\notin \E} 1-p_i$.
Clearly, each~$p_i$ and each~$(1 - p_i)$ is computable in \#P,
as they are given as part of the input.
Then, since \#P is closed under multiplication (over a polynomial number of \#P functions),
we get that~$\delta(\E)$ is in \#P as well.

Next, let us turn to~$\delta_{\topology{}}(f,T)$.
Remember that $\delta_{\topology{}}(f,T) = \sum_{\E:f(\E)=T}\delta(\E)$
if~$T \in \topology{\calE}$, and~$\delta_{\topology{}}(f,T) = 0$ otherwise.
Consider the function~$\delta'$ such that~$\delta'(\E) = \delta(\E)$
if~$f(\E) = T$ and such that~$\delta'(\E) = 0$ otherwise.
Then, because~$f$ is a polynomial-time computable function
and because~$\delta(\E)$ is in \#P, we know that~$\delta'$ is also in \#P.
Moreover, whenever~$T \in \topology{\calE}$, it holds
that~$\delta_{\topology{}}(f,T) = \sum_{\E}\delta'(\E)$.
Then, because~$f$ is a polynomial-time computable function,
because~$\delta'$ is in \#P,
and because \#P is closed under addition (over an exponential number of \#P functions),
we can conclude that~$\delta_{\topology{}}(f,T)$ is in \#P as well.

Next, consider~$\delta_{\calJ{}}(f,A)$.
Remember that $\delta_{\calJ{}}(f,A) =
\nicefrac{\delta_{\topology{}}(f,A)}{\sum_{T \in \calJ{}} \delta_{\topology{}}(f,T)}$
if~$A \in \calJ{}$ and~$\delta_{\calJ{}}(f,A) = 0$ otherwise.
Because \#P is closed under addition (over an exponential number of \#P functions),
by a similar argument as we used above,
because the characteristic function of $\calJ{}$ is polynomial-time computable,
we know that~$\sum_{T \in \calJ{}} \delta_{\topology{}}(f,T)$ is in \#P.
Then, because \#P is closed under division (of two \#P functions),
we can conclude that~$\delta_{\calJ{}}(f,A)$ is in \#P as well.

Finally, let us look at~$\Bel_{\calJ{}}(f,P)$.
Remember that $\Bel_{\calJ{}}(f,P) = \sum_{ A \subseteq P}\delta_{\calJ{}}(f,A)$.
Because \#P is closed under addition (over an exponential number of \#P functions),
by a similar argument as we used above,
we can conclude that~$\Bel_{\calJ{}}(f,P)$ is in \#P. This concludes our proof that \textsc{Degree of Belief} is in \#P.
\end{proof}
}

\paragraph{Lower bound}
Next, we show that the upper bound of \#P-membership is matched
by a \#P-hardness lower bound, even for a particular case
where we use a fixed frame of justification
and a fixed evidence allocation function.
In fact, this is the case that boils down to Dempster's rule of combination
(see Proposition~\ref{prop:generalization-dst}) - %
which we can use to straightforwardly establish \#P-hardness.

\begin{proposition}
\label{prop:sharp-p-hardness}
\textsc{Degree of Belief} is \#P-hard,
even when we require that the frame of justification
is~$\calJ{DS}$
and that the evidence allocation function
is the function~$i$ as defined in Proposition~\ref{prop:intersection}.
\end{proposition}
\begin{proof}
Consider the case where the frame of justification
is~$\calJ{DS}$
and where the evidence allocation function
is the function~$i$ as defined in Proposition~\ref{prop:intersection}.
We will show that \textsc{Degree of Belief} is \#P-hard,
even under these restrictions.
By Proposition~\ref{prop:generalization-dst},
we know that in this case, \textsc{Degree of Belief}
boils down to computing the belief~$\Bel(P)$ of a proposition~$P$
based on applying Dempster's rule of combination
to a given set of simple support functions.
This problem is known to be \#P-complete
\cite[Theorem~3.1]{PintoPrietoDeHaan22},
and thus \#P-hardness of \textsc{Degree of Belief}
follows directly.
\end{proof}

From this we can conclude that the problem in its most general form
is \#P-complete.

\begin{corollary}
\textsc{Degree of Belief} is \#P-complete.
\end{corollary}


\section{Conclusions and Future Research}\label{sec:conclusions}
In this paper, we proposed a new model for measuring degrees of beliefs based on possibly inconsistent, incomplete, and uncertain evidence.
We did so by combining tools from Dempster-Shafer Theory (DST)
and Topological Models of Evidence (TME),
yielding a model that is strictly more general than models from both approaches
(but whose worst-case complexity is not higher).

Future research includes extending our model and its analysis in various ways - among
others in the following directions.
It would be interesting to add notions of \emph{preference} to the evidence allocation functions,
that are based on more than just the evidence
(i.e., preferences that the agents might have whose source is external to the evidence).
Additionally, one could investigate whether (and how) our model could provide \emph{explanations}
for the values of the resulting belief functions.
To illustrate this, consider the (extreme) example where the pieces of evidence are (pairwise) disjoint.
In this case, the frame of justification $\calJ{SD}$ will return full certainty in the total set - i.e.,
it will yield a belief function that expresses that no belief should be attributed to any non-trivial proposition.
In practical scenarios, it would of course be useful to indicate the reason behind why
the belief function points to this particular conclusion - %
and similar explanations would be useful also for less extreme cases.
Moreover, it would be interesting and worthwhile to devise a \emph{logic}
(or multiple logics, based on different instantiations of our model)
to capture and study the properties of belief functions that the model yields.
Another interesting direction would be to consider \emph{frames of justification}
that are \emph{intermediate} (between the ones considered in this paper).
An example of this could be a weak denseness variant, where
only consistency with the basic pieces of evidence is required
(rather than consistency with all arguments in the topology).
Finally, we mention the direction of undertaking a \emph{more detailed analysis of the
computational complexity} properties of our model.

\section*{Acknowledgments}

We would like to thank the members of the Heudiasyc Laboratory at the University of Technology of Compiègne, especially Sébastien Destercke, Xu Philippe, and Thierry Denœux, for their insightful comments and inputs, which helped to improve the quality of this paper. Their expertise in Dempster-Shafer Theory and applications was fundamental in shaping the details of this research. 

Aybüke Özgün acknowledges support from the project "Responsible Artificial Agency: A Logical Perspective”, funded by the seed grant of the Research Priority Area Human(e) AI at the University of Amsterdam.

\bibliographystyle{kr}
\bibliography{kr23}

\end{document}